\documentclass[twoside,leqno,twocolumn]{article}
\usepackage{ltexpprt}
\usepackage{amsmath}
\usepackage{amsfonts}
\usepackage{graphicx}
\usepackage{subfig}
\usepackage{algorithm}
\usepackage{algorithmic}
\usepackage{verbatim}
\newtheorem{definition}{Definition}

\begin{document}

\title{\Large Large Scale Spectral Clustering \\ Using Approximate Commute Time Embedding}
\author{Nguyen Lu Dang Khoa \thanks{School of IT, University of Sydney, Australia. Email: khoa@it.usyd.edu.au}  \\
\and
Sanjay Chawla\thanks{School of IT, University of Sydney, Australia. Email: sanjay.chawla@sydney.edu.au}}

\date{}

\maketitle


\begin{abstract} \small\baselineskip=9pt
Spectral clustering is a novel clustering method which can detect complex shapes of data clusters. However, it requires the eigen decomposition of the graph Laplacian matrix, which is proportion to $O(n^3)$ and thus is not suitable for large scale systems. Recently, many methods have been proposed to accelerate the computational time of spectral clustering. These approximate methods usually involve sampling techniques by which a lot information of the original data may be lost. In this work, we propose a fast and accurate spectral clustering approach using an approximate commute time embedding, which is similar to the spectral embedding. The method does not require using any sampling technique and computing any eigenvector at all. Instead it uses random projection and a linear time solver to find the approximate embedding. The experiments in several synthetic and real datasets show that the proposed approach has better clustering quality and is faster than the state-of-the-art approximate spectral clustering methods.

\end{abstract}

\section{Introduction}
\label{chapter:intro}
Data clustering is an important problem and has been studied extensively in data mining research \cite{jain1999}. Traditional methods such as $k$-means or hierarchical techniques usually assume that data has clusters of convex shapes so that using Euclidean distance they can linearly separate them. On the other hand, spectral clustering can detect clusters of more complex geometry and has been shown to be more effective than traditional techniques in different application domains \cite{ng2001,shi2000,luxburg2007}. The intuition of spectral clustering is that it maps the data in the original feature space to the eigenspace of the Laplacian matrix where we can linearly separate the clusters and thus the clusters are easier to be detected using traditional techniques like $k$-means. This technique requires the eigen decomposition of the graph Laplacian which is proportional to $O(n^3)$ and is not applicable for large graphs.

Recent studies try to solve this problem by accelerating the eigen decomposition step. They either involves sampling or low-rank matrix approximation techniques \cite{fowlkes2004,wang2009,yan2009,chen2011}. \cite{fowlkes2004} used traditional Nystr\"{o}m method to solve the eigensystem solution on data representatives which were sampled randomly and then extrapolated the solution for the whole dataset. \cite{yan2009} performed the spectral clustering on a small set of data centers chosen by $k$-means or a random projection tree.  Then all data points were assigned to clusters corresponding to its centers in the center selection step. A recent work in \cite{chen2011} used the idea of sparse coding to approximate the affinity matrix based on a number of data representatives so that they can compute the eigensystem very efficiently. However, all of them involve sampling techniques. Although the samples or representatives are chosen uniformly at random or by using a more expensive selection, it may not completely represent the whole dataset and may not correctly capture the cluster geometry structures. Moreover, all of them involves computing the eigenvectors of the Laplacian and cannot be used directly in graph data which are popularly available such as social networks, web graphs, and collaborative filtering graphs.

In this paper, we propose a different approach using an approximate commute time embedding. Commute time is a random walk based metric on graphs. The commute time between two nodes $i$ and $j$ is the expected number of steps a random walk starting at $i$ will take to reach $j$ for the first time and then return back to $i$. The fact that commute time is averaged over all paths (and not just the shortest path) makes it more robust to data perturbations. Commute time has found widespread applications in personalized search \cite{sarkar2008}, collaborative filtering \cite{brand2005,fouss2007}, anomaly detection \cite{khoa2010}, link prediction in social network \cite{liben-nowell2003}, and making search engines robust against manipulation \cite{hopcroft2007}. Commute time can be embedded in an eigenspace of the graph Laplacian matrix where the square pairwise Euclidean distances are the commute time in the similarity graph \cite{fouss2007}. Therefore, the clustering using commute time embedding has similar idea to spectral clustering and they have quite similar clustering capability.

Another kind of study in \cite{mavroeidis2010} proposed a semi-supervised framework using data labels to improve the efficiency of the power method in finding eigenvectors for spectral clustering. Alternatively, \cite{chen2010} used parallel processing to accelerate spectral clustering in a distributed environment. In our work, we only focus on the acceleration of spectral clustering using a single machine in an unsupervised manner.

The contributions of this paper are as follows:
\begin{itemize}
    \item We show the similarity in idea and implementation between spectral clustering and clustering using commute time embedding. The experiments show that they have quite similar clustering capabilities.
    \item We show the weakness of sampling-based approximate approaches and propose a fast and accurate spectral clustering method using approximate commute time embedding. This does not sample the data, does not compute any eigenvector, and can work directly in graph data. Moreover, the approximate embedding can be applied to different other applications which utilized the commute time.
    \item We show the effectiveness of the proposed methods in terms of accuracy and performance in several synthetic and real datasets. It is more accurate and faster than the state-of-the-art approximate spectral clustering methods.
\end{itemize}

The remainder of the paper is organized as follows. Sections \ref{chapter:spectral} and \ref{chapter:related} describe the spectral clustering technique and efforts to approximate it to reduce the computational time. Section \ref{chapter:CTD} reviews notations and concepts related to commute time and its embedding, and the relationship between spectral clustering and clustering using commute time embedding. In Section \ref{chapter:CDST}, we present a method to approximate spectral clustering with an approximate commute time embedding. In Section \ref{chapter:expres}, we evaluate our approach using experiments on several synthetic and real datasets. Sections \ref{chapter:discussion} covers the discussion of the related issues. We conclude in Section \ref{chapter:conclusion} with a summary and a direction for future research.

\section{Spectral Clustering}
\label{chapter:spectral}

Given a dataset $X \in \mathbb{R}^{d}$ with $n$ data points $x_1, x_2, \ldots, x_n$ and $d$ dimensions, we define an undirected and weighted graph $G$. Let $A=w_{ij} (1 \leq i,j \leq n)$ be the affinity matrix of $G$.

Let $i$ be a node in $G$ and $N(i)$ be its neighbors. The {\it degree} $d_{i}$ of a node $i$ is $\sum_{j \in N(i)}w_{ij}$. The {\it volume} $V_{G}$ of the graph is defined as $\sum_{i=1}^{n}d_{i}$. Let $D$ be the diagonal degree matrix with diagonal entries $d_i$. The Laplacian of $G$ is the matrix $L = D - A$.

Spectral clustering assigns each data point in $X$ to one of $k$ clusters. The details are in Algorithm \ref{algo:spectral}.

\begin{algorithm}[h!]
\caption{Spectral Clustering}
\label{algo:spectral}
\textbf{Input:} Data matrix $X \in \mathbb{R}^{d}$, number of clusters $k$\\
\textbf{Output:} Cluster membership for each data point \\
\begin{algorithmic}[1]
\STATE Construct a similarity graph $G$ from $X$ and compute its Laplacian matrix $L$ \\
\STATE Compute the first $k$ eigenvectors of $L$. \\
\STATE Let $U \in \mathbb{R}^{k}$ be the eigenspace containing these $k$ vectors as columns and each row of $U$ corresponds to a data point in $X$.\\
\STATE Cluster the points in $U$ using $k$-means clustering.
\end{algorithmic}
\end{algorithm}

There are three typical similarity graphs: the $\varepsilon$-neighborhood graph (connecting nodes whose distances are shorter than $\varepsilon$), the fully connected graph (connecting all nodes with each other), and the $k$-nearest neighbor graph (connecting nodes $u$ and $v$ if $u$ belongs to $k$ nearest neighbors of $v$ \emph{or} $v$ belongs to $k$ nearest neighbors of $u$) \cite{luxburg2007}. The $\varepsilon$-neighborhood graph and $k$-nearest neighbor graph ($k \ll n$) are usually sparse, which have advantages in computation. The typical similarity function is the Gaussian kernel function $w_{ij} = e^{-\frac{\|x_i-x_j\|^2}{2\sigma^2}}$ where $\sigma$ is the kernel bandwidth.

Algorithm \ref{algo:spectral} shows that spectral clustering transforms the data from its original space to the eigenspace of the Laplacian matrix and uses $k$-means to cluster data in that space. The representation in the new space enhances the cluster properties in the data  so that the clusters can be linearly separated \cite{luxburg2007}. Therefore, traditional technique like $k$-means can easily cluster data in the new space.

We can use the normalized Laplacian matrix and its corresponding eigenvectors as the eigenspace. Shi and Malik \cite{shi2000} computed the first $k$ eigenvectors of the generalized eigensystem as the eigenspace. These eigenvectors are in fact the eigenvectors of the normalized Laplacian $L_n=D^{-1}L$ \cite{luxburg2007}. Ng, Jordan, and Weiss use $k$ eigenvectors of the normalized Laplacian $L_n=D^{-1/2}LD^{-1/2}$ as the eigenspace. It then requires the normalization of each row in the new space to norm 1 \cite{ng2001}.

\section{Related Work}
\label{chapter:related}
The spectral clustering method described in previous section involves the eigen decomposition of the (normalized) Laplacian matrix. It takes $O(n^3)$ time and is not feasible to do for large graphs. Even if we can reduce it by using a sparse similarity graph and a sparse eigen decomposition algorithm with an iterative approach, it is still expensive for large graphs.

Most of the approaches try to approximate spectral clustering using sampling or low-rank approximation techniques. \cite{fowlkes2004} used Nystr\"{o}m technique and \cite{wang2009} used column sampling to solve the eigensystem in a smaller sample and extrapolated the solution for the whole dataset.

\cite{yan2009} provided a framework for a fast approximate spectral clustering. A number of centers were chosen from the data by using $k$-means or a random projection tree. Then these centers were clustered by the spectral clustering. The cluster membership for each data point corresponding to its center was assigned using the spectral clustering membership in the center set. However, the center selection step is time consuming for large datasets.

\cite{chen2011} used the idea of sparse coding to design an approximate affinity matrix $A = ZZ^T (Z \in \mathbb{R}^{s}$ where $s$ is the number of representatives, or landmarks in their word) so that the eigen decomposition of an $(n \times n)$ matrix $A$ can be found from the eigen decomposition of a smaller $(s \times s)$ matrix $Z^TZ$. Since the smallest eigenvectors of $L_n=D^{-1/2}LD^{-1/2}$ are the largest eigenvectors of $D^{-1/2}AD^{-1/2}$, we have the eigen solution of $L_n$.  $s$ landmarks can be selected by random sampling or by $k$-means method. They claimed that choosing the landmarks by randomly sampling is a balance between accuracy and performance.

However, all these methods involve data sampling either by choosing randomly or by a $k$-means selection. Using $k$-means or other methods to select the representative centers is costly in large datasets since the number of representatives cannot be too small. Moreover, any kind of sampling will suffer from the lost of information in the original data since the representatives may not completely represent the whole dataset and may not correctly capture the cluster geometry structures. Therefore, any approximation based on these representatives also suffers from this information lost. These facts will be illustrated in the experiments. Moreover, these approximations cannot be used directly for graph data.

\section{Commute Time Embedding and Spectral Clustering}
\label{chapter:CTD}

This section reviews the concept of the commute time, its embedding and the relationship between clustering in the commute time embedding and spectral clustering.

\begin{definition} The Hitting Time $h_{ij}$ is the expected number of steps that a random walk starting at $i$ will take before reaching $j$ for the first time.
\end{definition}

\begin{definition}The Hitting Time can be defined in terms of the recursion
\begin{displaymath}
h_{ij} = \begin{cases}
1 + \sum_{l \in N(i)}p_{il}h_{lj} & \mbox{ if $ i \neq j$} \\
0 & \mbox{otherwise}
\end{cases}
\end{displaymath}
where
\begin{displaymath}
p_{ij} = \begin{cases}
w_{ij}/d_{i} & \mbox{ if $(i,j)$ belong to an edge} \\
0 & \mbox{ otherwise }
\end{cases}.
\end{displaymath}
\end{definition}

\begin{definition} The Commute Time $c_{ij}$  between two nodes $i$ and $j$ is given by $c_{ij} = h_{ij} + h_{ji}$.
\end{definition}

\begin{fact} Commute time is a metric: (i) $c_{ii} = 0$, (ii) $c_{ij} = c_{ji}$ and (iii) $c_{ij} \leq c_{ik} + c_{kj}$ \cite{klein1993}.
\end{fact}

\begin{fact}
\begin{enumerate}
\item
Let $e_{i}$ be the $n$ dimensional column vector with a 1 at location $i$ and zero elsewhere.
\item Let $(\lambda_{i}, v_{i})$ be the eigenpair of $L$ for all nodes $i$, i.e., $Lv_{i} = \lambda_{i}v_{i}$.
\item It is well known that $\lambda_{1}=0, v_{1} =(1,1,\ldots,1)^{T}$ and all $\lambda_{i} \geq 0$.
\item Assume $0 = \lambda_{1} \leq \lambda_{2} \ldots \leq \lambda_{n}$.
\item The eigen decomposition of the Laplacian is $L=VSV^T$ where $S=diag(\lambda_{1}, \lambda_{2}, \ldots, \lambda_{n})$ and $V=(v_1, v_2, \ldots, v_n)$.
\item Then the pseudo-inverse of $L$ denoted by $L^{+}$ is
\[
L^{+} = \sum_{i=2}^{n}\frac{1}{\lambda_{i}}v_{i}v_{i}^{T}
\]
\end{enumerate}
\end{fact}

Remarkably, the commute time can be expressed in terms of the Laplacian of $G$ \cite{fouss2007,doyle1984}.
\begin{fact}
\begin{equation}
\label{equa1}
c_{ij} = V_{G}(l_{ii}^{+}+l_{jj}^{+}-2l_{ij}^{+}) = V_{G}(e_{i}-e_{j})^{T}L^{+}(e_{i}-e_{j})
\end{equation}
where $l_{ij}^{+}$ is the $(i,j)$ element of $L^{+}$ \cite{fouss2007}.
\end{fact}

\begin{theorem}
$\theta = \sqrt{V_G}VS^{-1/2} \in \mathbb{R}^{n}$ is a commute time embedding where the square root of the commute time $\sqrt{c_{ij}}$ is an Euclidean distance between $i$ and $j$ in $\theta$.
\end{theorem}

\begin{proof}
Equation \ref{equa1} can be written as:
\begin{align*}
c_{ij}  &= V_{G}(e_{i}-e_{j})^{T}L^{+}(e_{i}-e_{j}) \\
        &= V_{G}(e_{i}-e_{j})^{T}V S^{-1} V^{T}(e_{i}-e_{j})  \\
        &= V_{G}(e_{i}-e_{j})^{T}VS^{-1/2}S^{-1/2}V^{T} (e_{i}-e_{j}) \\
        &= [\sqrt{V_G}S^{-1/2}V^{T}(e_{i}-e_{j})]^{T}[\sqrt{V_G}S^{-1/2}V^{T}(e_{i}-e_{j})].
\end{align*}

Thus the commute time is the square pairwise Euclidean distance between column vectors in space {$\sqrt{V_G}S^{-1/2}V^{T}$} or row vectors in space $\theta = \sqrt{V_G}VS^{-1/2}$.
\end{proof}

Since the commute time can capture the geometry structure in the data, using $k$-means in the embedding $\theta$ can effectively capture the complex clusters. This is very similar to the idea of spectral clustering. The commute time is a novel metric capturing the data geometry structure and is embedded in the Laplacian eigenspace. Alternatively, spectral clustering maps the original data to the eigenspace of the Laplacian where the clusters can be linearly separated. However, there are some differences between Commute time Embedding Spectral Clustering (denoted as CESC) and spectral clustering.
\begin{itemize}
\item Spectral clustering only uses $k$ eigenvectors of $V$ while CESC uses all the eigenvectors.
\item The eigenspace in CESC is scaled by the eigenvalues of the Laplacian.
\end{itemize}

In case of the normalized Laplacian $L_n=D^{-1/2}LD^{-1/2}$, the embedding for the commute time is $\theta_n = \sqrt{V_G}D^{-1/2}V_nS_n^{-1/2}$ \cite{qiu2007} where $V_n$ and $S_n$ are the matrix containing eigenvectors and eigenvalues of $L_n$. The normalized eigenspace is scaled more with the degree matrix. However, since the commute time is a metric independent to the Laplacian and $k$-means in the eigenspace uses the square Euclidean distance which is the commute time in the graph space, CESC is independent to the use of the normalized or unnormalized Laplacian.

\section{Approximate Commute Time Embedding Clustering}
\label{chapter:CDST}
The embedding $\theta = \sqrt{V_G}VS^{-1/2}$ is costly to create since it take $O(n^3)$ for the eigen decomposition of $L$. Even if we can make use the sparsity of $L$ in sparse graph by computing a few smallest eigenvectors of $L$ \cite{saerens2004b} using Lanczos method \cite{golub1996}, the method is still hard to converge and thus is inefficient for large graphs. We adopt the idea in \cite{spielman2008} to approximate the commute time embedding more efficiently. Speilman and Srivastava \cite{spielman2008} used random projection and the linear time solver of Speilman and Teng \cite{spielman2004,spielman2006} to build a structure where we can compute the compute time between two nodes in $O(\log{n})$ time.

\begin{fact}
Let $m$ be the number of edges in $G$. If the edges in $G$ are oriented, $B_{m \times n}$ given by:
\begin{displaymath}
B(u,v) = \begin{cases}
            1 & \mbox{if $v$ is $u$'s head} \\
           -1 & \mbox{if $v$ is $u$'s tail} \\
            0 & \mbox{otherwise}
                \end{cases}
\end{displaymath}
is a signed edge-vertex incidence matrix and $W_{m \times m}$ is a diagonal matrix whose entries are the edge weights. Then $L=B^{T}WB$.
\end{fact}

\begin{lemma}
\label{lem:CDST1}
(\cite{spielman2008}). $\theta = \sqrt{V_G}L^{+}B^{T}W^{1/2} \in \mathbb{R}^{m}$ is a commute time embedding where the square root of the commute time $\sqrt{c_{ij}}$ is an Euclidean distance between $i$ and $j$ in $\theta$.
\end{lemma}
\begin{proof}
From Equation \ref{equa1}:
\begin{align*}
c_{ij}  &= V_G(e_{i}-e_{j})^{T}L^{+}(e_{i}-e_{j}) \\
        &= V_G(e_{i}-e_{j})^{T}L^{+}LL^{+}(e_{i}-e_{j}) \\
        &= V_G(e_{i}-e_{j})^{T}L^{+}B^{T}WBL^{+}(e_{i}-e_{j}) \\
        &= V_G[(e_{i}-e_{j})^{T}L^{+}B^{T}W^{1/2}][W^{1/2}BL^{+}(e_{i}-e_{j})] \\
        &= [\sqrt{V_G}W^{1/2}BL^{+}(e_{i}-e_{j})]^{T}[\sqrt{V_G}W^{1/2}BL^{+}(e_{i}-e_{j})]
\end{align*}
Thus the commute time is the square pairwise Euclidean distance between column vectors in the space {$\sqrt{V_G}W^{1/2}BL^{+}$} or between row vectors in space $\theta = \sqrt{V_G}L^{+}B^{T}W^{1/2} \in \mathbb{R}^{m}$.
\end{proof}

These distances are preserved under the Johnson-Lindenstrauss Lemma if we project a row vector in $\theta$ onto a subspace spanned by $k_{RP} = O(\log{n})$ random vectors \cite{johnson1984}. We can use a random matrix $Q_{k_{RP} \times m}$ where $Q(i,j) = \pm1/\sqrt{k_{RP}}$ with equal probabilities regarding the following lemma.

\begin{lemma}
\label{lem:CDST2}
(\cite{achlioptas2001}). Given fix vectors $v_1,...,v_n \in \mathbb{R}^{d}$ and $\epsilon>0$, let $Q_{k_{RP} \times d}$ be a random matrix so that $Q(i,j) = \pm1/\sqrt{k_{RP}}$ with $k_{RP}=O(\log{n}/\epsilon^2)$. With probability at least $1-1/n$:
\begin{displaymath}
(1-\epsilon)\|v_i-v_j\|^2 \leq \|Qv_i-Qv_j\|^2 \leq (1+\epsilon)\|v_i-v_j\|^2
\end{displaymath}
for all pairs $i,j \in G$.
\end{lemma}

\begin{theorem}
(\cite{spielman2008}). Given $\epsilon>0$ and a matrix $Z_{O(\log{n}/\epsilon^2) \times n}=\sqrt{V_G}QW^{1/2}BL^{+}$, with probability at least $1-1/n$:
\begin{displaymath}
(1-\epsilon)c_{ij} \leq \|Z(e_i-e_j)\|^2 \leq (1+\epsilon)c_{ij}
\end{displaymath}
for all pairs $i,j \in G$.
\end{theorem}
\begin{proof}
The proof comes directly from Lemmas \ref{lem:CDST1} and \ref{lem:CDST2}.
\end{proof}

Therefore we are able to construct a matrix $Z=\sqrt{V_G}QW^{1/2}BL^{+}$ which $c_{ij} \approx \|Z(e_i-e_j)\|^2$ with an error $\epsilon$. Since to compute $L^{+}$ directly is expensive, the linear time solver of Spielman and Teng \cite{spielman2004,spielman2006} is used instead.  First, $Y=\sqrt{V_G}QW^{1/2}B$ is computed. Then each of $k_{RP}=O(\log{n})$ rows of $Z$ (denoted as $z_i$) is computed by solving the system $z_iL=y_i$ where $y_i$ is a row of $Y$. The linear time solver of Spielman and Teng takes only $\widetilde{O}(m)$ time to solve the system \cite{spielman2008}.

Since $\|z_i-\widetilde{z}_i\|_L \leq \epsilon \|z_i\|_L$ where $\widetilde{z}_i$ is the solution of $z_iL=y_i$ using the linear time solver \cite{spielman2008} we have:
\begin{equation}
\label{equa2}
(1-\epsilon)^2c_{ij} \leq \|\widetilde{Z}(e_i-e_j)\|^2 \leq (1+\epsilon)^2c_{ij}
\end{equation}
where $\widetilde{Z}$ is the matrix containing row vector $\widetilde{z}_i$.

Equation \ref{equa2} shows that the approximate spectral clustering using approximate commute time embedding by combining random projection and a linear time solver has the error $\epsilon^2$. The method is detailed in Algorithm \ref{algo:CESC}.

\begin{algorithm}[h!]
\caption{Commute time Embedding Spectral Clustering (CESC)}
\label{algo:CESC}
\textbf{Input:} Data matrix $X \in \mathbb{R}^{d}$, number of clusters $k$, number of random vectors $k_{RP}$\\
\textbf{Output:} Cluster membership for each data point \\
\begin{algorithmic}[1]
\STATE Construct a $k_1$-nearest neighbor graph $G$ from $X$ with Gaussian kernel similarity ($k_1 \ll n$). \\
\STATE Compute matrices $B$, $W$, and $L$ from $G$. \\
\STATE Compute $Y=\sqrt{V_G}QW^{1/2}B$ where $Q$ is an $\pm1/\sqrt{k_{RP}}$ random matrix. \\
\STATE Compute all rows $\widetilde{z}_i$ of $\widetilde{Z}_{k_{RP} \times n}=YL^+$ by $k_{RP}$ calls to the Spielman-Teng solver.\\
\STATE Cluster the points in $\widetilde{Z}^T$ using $k$-means clustering.
\end{algorithmic}
\end{algorithm}

In Algorithm \ref{algo:CESC}, $\theta = \widetilde{Z}^{T} \in \mathbb{R}^{k_{RP}=O(\log{n})}$ is the embedding space where the square pair wise Euclidean distance is the approximate commute time. Applying $k$-means in $\theta$ is a novel way to accelerate spectral clustering without using any sampling technique and computing any eigenvector. Moreover, the approximate embedding is guaranteed with the error bound $\epsilon^2$ and the method can be applied directly in graph data.

\subsection{Analysis}
\label{section:complexity}
Here we analyze the computational complexity of the proposed method. Firstly the $k_1$-nearest neighbor graph is constructed in $O(n\log{n})$ time using $kd$-tree. $Y=\sqrt{V_G}QW^{1/2}B$ is computed in $O(2mk_{RP} + m) = O(mk_{RP})$ time since there are only $2m$ nonzeros in $B$ and $W$ is a diagonal matrix with $m$ nonzeros. Then each of $k_{RP}$ rows of $\widetilde{Z}$ (denoted as $\widetilde{z}_i$) is computed by solving the system $z_iL=y_i$ in $\widetilde{O}(m)$ time where $y_i$ is a row of $Y$. Since we use $k_1$-nearest neighbor graph where $k_1 \ll n$, $O(m)=O(n)$. Therefore, the construction of $\widetilde{Z}$ takes $\widetilde{O}(nk_{RP})$ time. $k$-means algorithm in $\widetilde{}Z^T$ takes $O(tkk_{RP}n)$ where $k$ is the number of clusters and $t$ is the number of iterations for the algorithm to be converged.

The summary of the analysis of CESC and other approximate spectral clustering techniques is in Table \ref{tab:table1}. All methods create the embedded space where they use $k$-means to cluster the data. The dimension of the embedding of Nystr\"{o}m, KASP, and LSC is $k$ - the number of clusters. For CESC, it is $k_{RP}$.

\begin{table*}
  \centering
  \caption{Complexity comparisons of all approximate spectral clustering methods. $n, d, s, k_{RP}, k$ is the number of instances, features, representatives, random projection columns, and the number of clusters, respectively.}
  \begin{tabular}{| l | l | l | l |l|}
    \hline
    Method      & Sampling  & Affinity matrix   & Embedded space            & $k$-means      \\
    \hline
    Nystr\"{o}m & $O(1)$    & $O(dsn)$          & $O(s^3 + sn)$             & $O(tk^2n)$     \\
    KASP        & $O(tdsn)$ & $O(ds^2)$         & $O(s^3)$                  & $O(tk^2s)$     \\
    LSC         & $O(1)$    & $O(dsn)$          & $O(s^3 + s^2n)$           & $O(tk^2n)$     \\
    CESC        & N/A       & $O(dn\log{n})$    & $\widetilde{O}(k_{RP}n)$  & $O(tkk_{RP}n)$ \\
    \hline
  \end{tabular}
  \label{tab:table1}
\end{table*}

Note that in practise, $k_{RP}=O(\log{n}/\epsilon^2)$ is small and does not have much differences between different datasets. We will discuss it in the experimental section. We can choose $k_{RP} \ll n$. Moreover, the performance of the linear time solver is observed to be linear empirically instead of $\widetilde{O}(m)$ \cite{koutis2009}. Therefore, the construction of $\widetilde{Z}$ takes only $O(nk_{RP})$ in practise.

On the contrary, the number of representatives $s$ cannot be very small in order to correctly represent the whole dataset. Therefore, the term $O(s^3)$ cannot be ignored. It is shown in the experiment that CESC is faster than all other methods while still maintaining better quality in clustering results.

\section{Experimental Results}
\label{chapter:expres}

\subsection{Evaluation criteria}
We report on the experiments carried out to determine and compare the effectiveness of the Nystr\"{o}m, KASP, LSC, and CESC methods. It included the clustering accuracy (percentage) and the computational time (second). For accuracy, it was measured against spectral clustering as the benchmark method since all of them are its approximations. The accuracy was computed by counting the fraction of matching between cluster memberships of spectral clustering and the approximate method, given by:
\begin{displaymath}
Accuracy = \frac{\sum_{i=1}^{n}\delta[map(c_i)=label(i)]}{n},
\end{displaymath}
where $n$ is the number of data instances, $label(i)$ and $c_i$ are the actual cluster label and the predicted label of a data instance $i$, respectively. $\delta(\cdot)$ is an indicator function and $map(c_i)$ is a permutation function that maps cluster $c_i$ to a category label. The best matching can be found using Hungarian algorithm \cite{chen2010}.

\subsection{Methods and Parameters}
All the experimental results reported in the following sections were the average over 10 trials. We chose Gaussian kernel function as the similarity function for all the methods. The bandwidth $\sigma$ was chosen based on the width of the neighborhood information for each dataset. For Nystr\"{o}m, KASP, and LSC, the eigenspace was created from the normalized Laplacian $L = D^{-1/2}LD^{-1/2}$ since the normalized one is reported to be better \cite{luxburg2004}. Methods using the nearest neighbor graph chose $k_1=10$ as the number of nearest neighbor in building the similarity graph.

The followings are the detailed information regarding the experiments for each method:

\textbf{$k$-means}: all the approximate methods used $k$-means to cluster the data in the embedding. The Matlab build-in function `kmeans' was used. The number of replications was 5 and the maximum number of iterations was 100. The `cluster' option (i.e. cluster 10\% of the dataset to choose initial centroids) was used.

\textbf{Spectral clustering}: we implemented in Matlab the algorithm in \cite{ng2001}. Since it is not possible to do the eigen decomposition of the Laplacian matrix in fully connected graph for large datasets, a $k_1$-nearest neighbor graph was built and the sparse function `eigs' was used to find the eigenspace.

\textbf{Nystr\"{o}m}: we used the Matlab implementation of Chen et. al \cite{chen2010} which is available online at \emph{http://alumni.cs.ucsb.edu/$\sim$wychen/sc.html}.

\textbf{KASP:} we implemented in Matlab the algorithm in \cite{yan2009} and used $k$-means to select the representative centers.

\textbf{LSC:} we used the number of $k_1=10$ for building the sparse matrix $Z$. In \cite{chen2011}, the representatives can be chosen by randomly sampling or by $k$-means. Since the random selection was preferred by the authors and had a better balance between running time and accuracy, we only used this option in the experiments.

\textbf{CESC}: the algorithm was implemented in Matlab. The number of random vectors $k_{RP}=50$ was chosen throughout the experiments. We used the Koutis's CMG solver \cite{koutis2009} as the nearly linear time solver for creating the embedding. It is used for symmetric diagonally dominant matrices which is available online at \emph{http://www.cs.cmu.edu/$\sim$jkoutis/cmg.html}.

\subsection{An example}
A synthetic dataset featured the data clusters in the shapes of a phrase `Data Mining' as in Figure \ref{fig:example}. It has 2,000 data points in 10 clusters. We applied CESC, Nystr\"{o}m, KASP, and LSC on this dataset. The number of representatives was 500 which was 25\% of the data. The results are shown in Figure \ref{fig:example}. In the figures of Nystr\"{o}m, KASP, and LSC, the red dots are the representatives selected in their corresponding methods.


\begin{figure*}
  \centering
  \subfloat[Nystr\"{o}m]{\label{fig:example_Nystrom2}\includegraphics[width=0.3\textwidth]{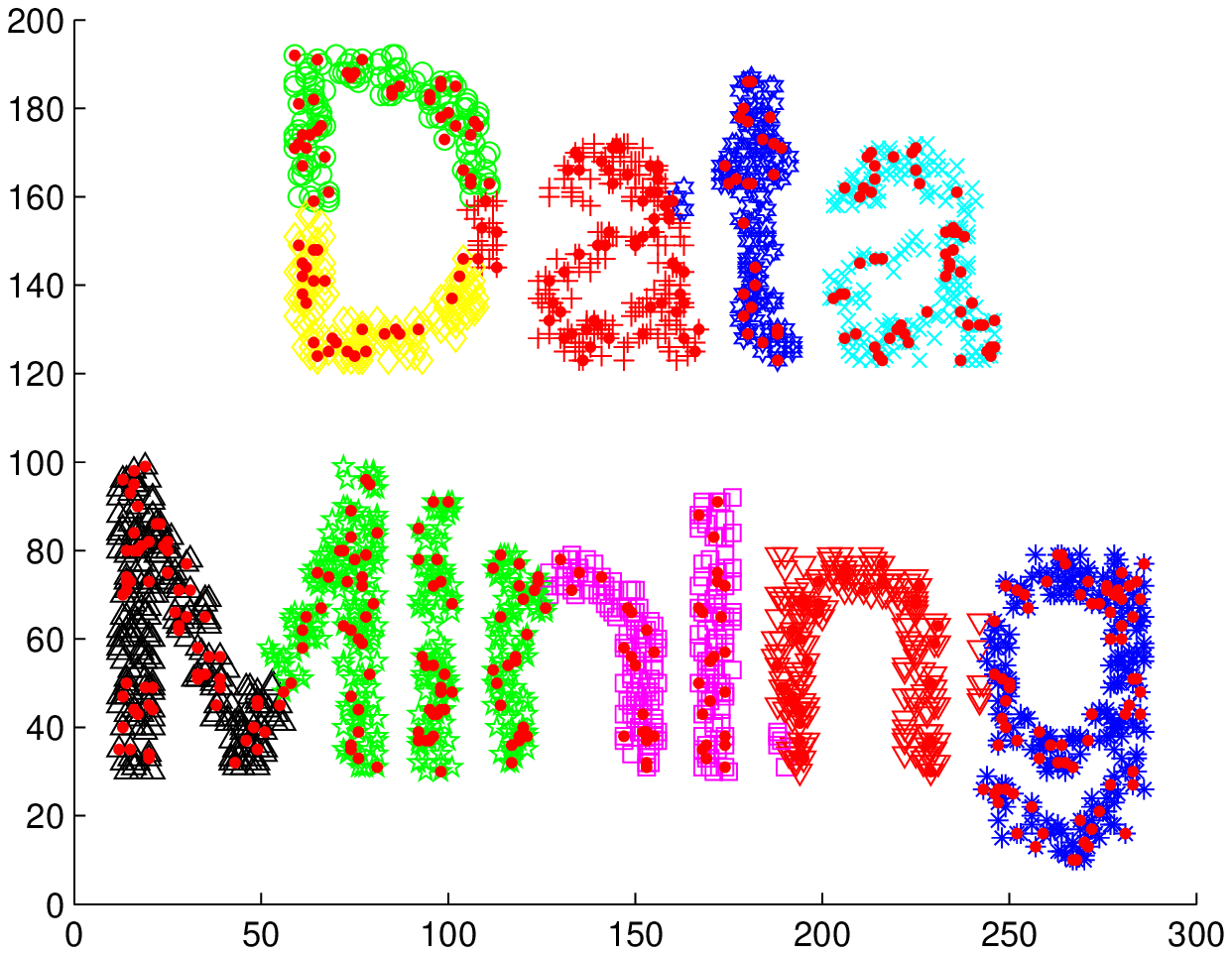}}
  \subfloat[KASP]{\label{fig:example_KASP2}\includegraphics[width=0.3\textwidth]{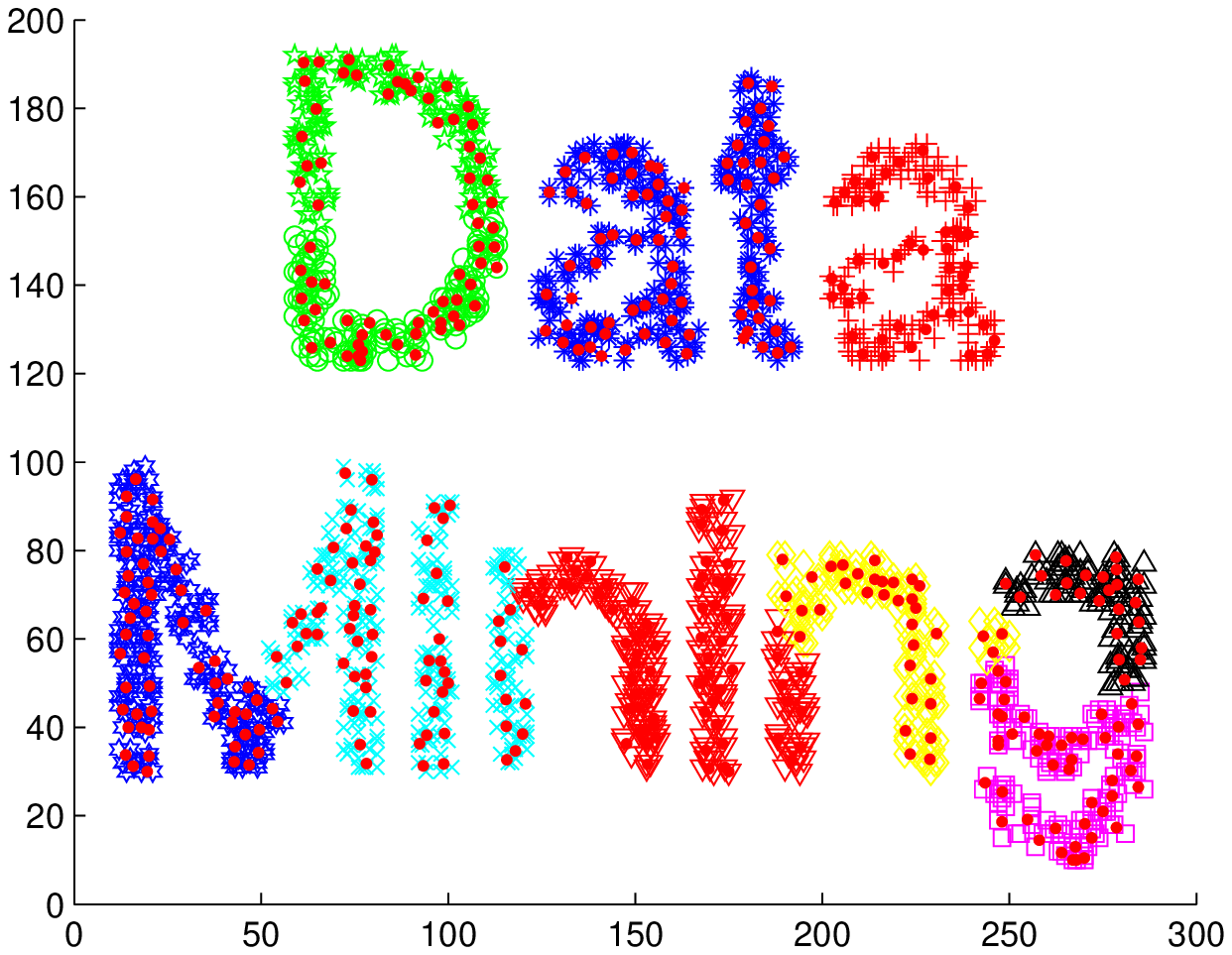}} \\
  \subfloat[LSC]{\label{fig:example_LSC2}\includegraphics[width=0.3\textwidth]{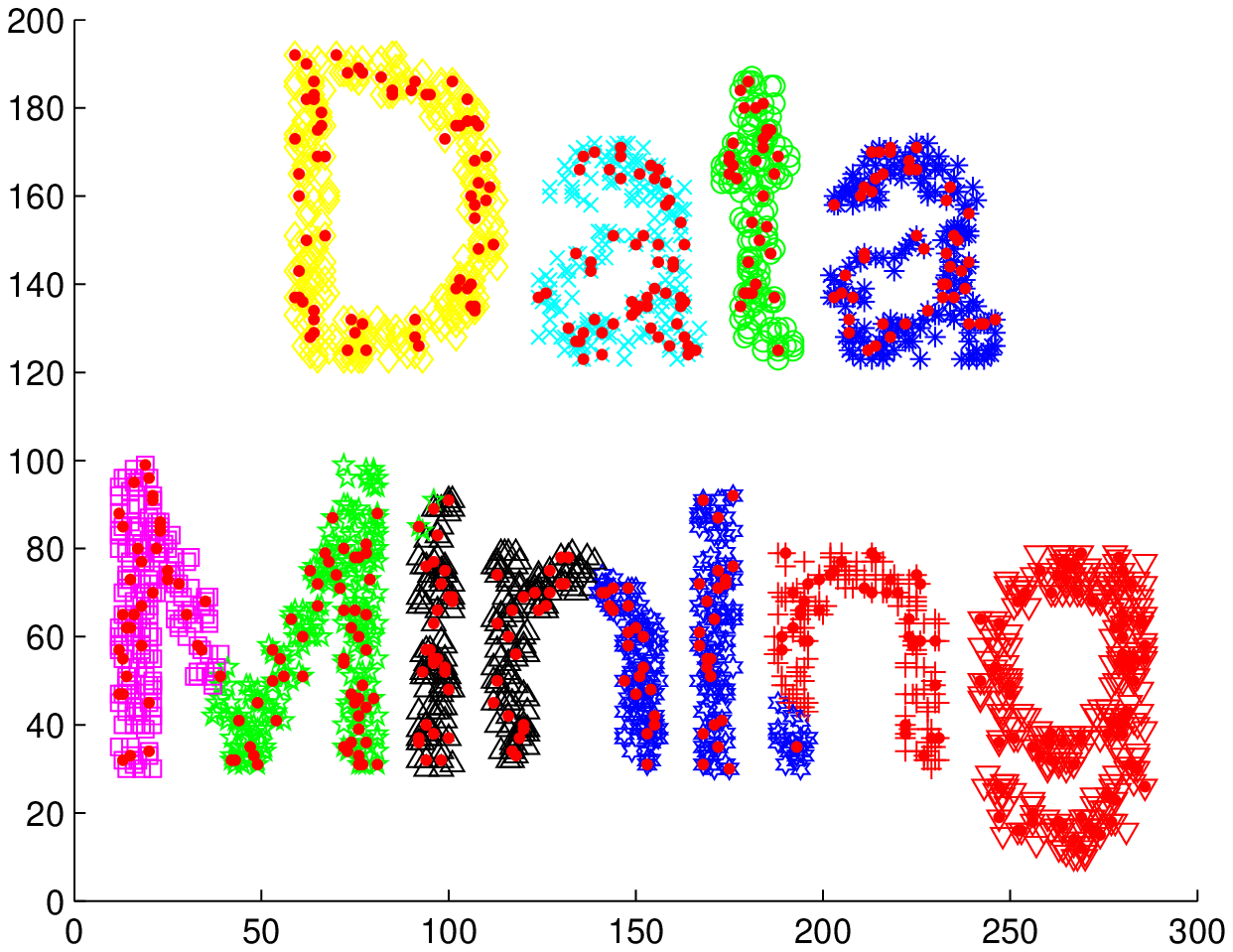}} 
  \subfloat[CESC]{\label{fig:example_CDST2}\includegraphics[width=0.3\textwidth]{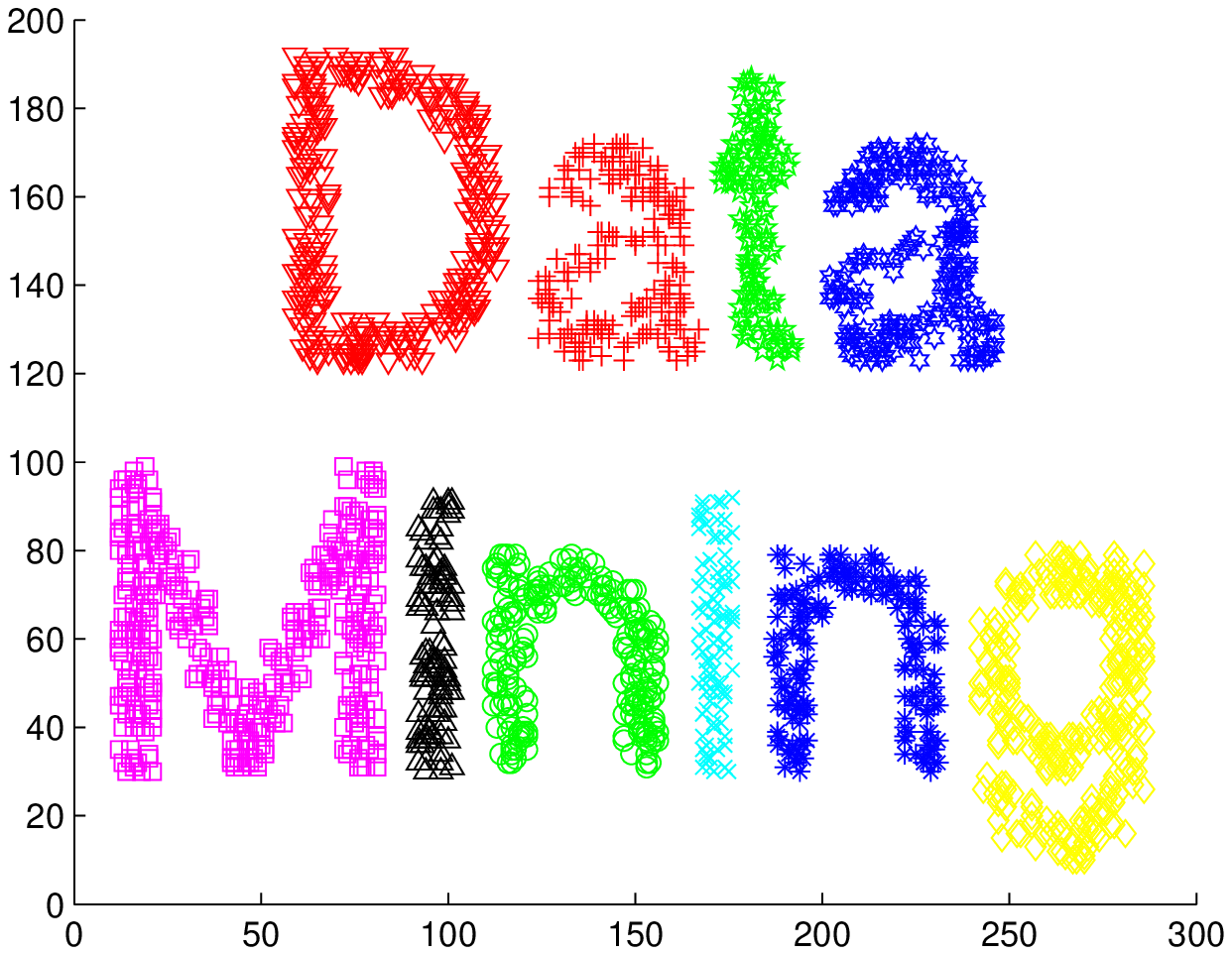}}
  \caption{Approximate spectral clustering methods using Nystr\"{o}m, KASP, LSC, and CESC. This shows the weakness of approximate methods based on sampling. The red dots are the representatives in Nystr\"{o}m, KASP, and LSC. CESC and exact spectral clustering can cluster the two datasets correctly.}
  \label{fig:example}
\end{figure*}

It can be seen from the results the weakness of sampling-based approximate methods and the strength of CESC. Although the number of representatives was large enough (25\% of data), it did not completely capture the geometry structures of all clusters and thus there were splits in some characters which a part of the character was considered closer to other character due to the structure of the representatives. CESC on the other hand clustered all data points correctly since it used all the data information. The exact spectral clustering also clustered the dataset correctly.

\subsection{Real Datasets}
We tested all the four methods in several real datasets with various sizes obtained from the UCI machine learning repository \cite{FrankAsuncion2010}. The details of all datasets are in Table \ref{tab:table2}. For all datasets, we normalized them so that all features had mean 0 and standard deviation 1.

\begin{table*}
  \centering
  \caption{UCI Datasets.}
  \begin{tabular}{| l | r | r | r |l|}
    \hline
    Dataset       & Instances & Features      & Classes & Description           \\
    \hline
    Segmentation  & 2,100     & 19            & 7       & Image segmentation	\\
    Spambase      & 4,601     & 57            & 2       & Spam prediction       \\
    Musk          & 6,598     & 166           & 2       & Molecule prediction   \\
    Pen Digits	  & 10,992    & 16            & 10      & Pen-based recognition of handwritten digits \\
    Letter Rec    & 20,000    & 16            & 26      & Letter recognition    \\
    Connect4      & 67,557    & 42            & 3       & The game of connect-4 \\
    \hline
  \end{tabular}
  \label{tab:table2}
\end{table*}

Regarding the number of representatives, it cannot be too small to truely represent the data or too big to significantly slower the methods. For the first 3 small datasets (Segmentation, Spambase, and Musk), we used 20\% of the data as the representatives. Medium sizes Pen Digits and Letter Rec used 10\% of the data as the representatives. For the big size dataset Connect4, we only chose 5,000 as the representatives. It is less than 10\% of data in Connect4. Since the computational time for large datasets will be very expensive for the sampling based methods if we use a high number of representatives, the percentage of the representatives will be less in larger datasets.

Tables \ref{tab:table3} and \ref{tab:table4} show the clustering results in accuracy (percentage) and running time (second) for all the datasets. Considering the accuracy, CESC outperformed all the sampling based approximate methods in most of datasets although the number of representatives they used was high enough. Considering the computational time, CESC was also the fastest method in all datasets.

\begin{table*}
  \centering
  \caption{Clustering accuracy (percentage). CESC outperformed other methods in most of the datasets.}
  \begin{tabular}{| l | r | r | r |r|}
    \hline
    Dataset       & KASP    & Nystr\"{o}m   & LSC               & CESC                \\
    \hline
    Segmentation  & 74.5\%  & 58.3\%        & 73.2\%            & \textbf{78.9\%}     \\
    Spambase      & 60.8\%  & 82.7\%        & 97.6\%            & \textbf{100\%}	  \\
    Musk          & 81.3\%  & 50.6\%        & 63.2\%            & \textbf{97.2\%}	  \\
    Pen Digits    & \textbf{83.4\%} & 74.8\%& 80.1\%            & 77.5\%              \\
    Letter Rec    & 52.8\%  & 39.2\%        & \textbf{58.5\%}   & 40.1\%              \\
    Connect4      & 86.8\%  & 35.3\%        & 83.0\%            & \textbf{97.4\%}     \\
    \hline
  \end{tabular}
  \label{tab:table3}
\end{table*}

\begin{table*}
  \centering
  \caption{Computational time (second). CESC was the fastest among all the approximate methods.}
  \begin{tabular}{| l | r | r | r |r|}
    \hline
    Dataset       & KASP    & Nystr\"{o}m   & LSC       & CESC              \\
    \hline
    Segmentation  & 2.26    & 6.25          & 8.87      & \textbf{2.06}	    \\
    Spambase      & 25.33   & 28.26         & 46.68     & \textbf{16.32}	\\
    Musk          & 178.24  & 110.87        & 154.09    & \textbf{63.18}	\\
    Pen Digits    & 65.33   & 104.01        & 119.04    & \textbf{12.46}    \\
    Letter Rec    & 236.43  & 529.86        & 395.47    & \textbf{59.45}    \\
    Connect4      & 3400.38 & 10997.14      & 3690.86   & \textbf{1839.59}  \\
    \hline
  \end{tabular}
  \label{tab:table4}
\end{table*}

From the complexity analysis in Section \ref{section:complexity}, we can see that the bottleneck of CESC is the total running time of graph creation and $k$-means steps. This is clearly shown in the results in Table \ref{tab:table5}, which presents the details in percentage of the running time of CESC for each dataset. The running time of the embedding step was dominated by the other two steps. The advantage is that there have been many studies in fast $k$-means and graph creation, or techniques to parallel them which we can make use of \cite{chen2010}. 

\begin{table}
  \centering
  \caption{Time distribution for CESC. The bottleneck of the algorithm is the total running time of graph creation and $k$-means steps.}
  \begin{tabular}{| l | r | r | r|}
    \hline
    Datasets        & Graph            & Embedding    & $k$-means   \\
    \hline
    Segmentation    & 54.0\%           & 31.1\%       & 14.9\%       \\
    Spambase        & 90.1\%           & 9.1\%        & 0.8\%       \\
    Musk            & 92.6\%           & 6.9\%        & 0.5\%       \\
    Pen Digits      & 51.1\%           & 33.0\%       & 15.8\%      \\
    Letter Rec      & 36.5\%           & 17.0\%       & 46.5\%      \\
    Connect4        & 97.1\%           & 2.7\%        & 0.2\%       \\
    \hline
  \end{tabular}
  \label{tab:table5}
\end{table}

\subsection{Parameter sensitivity}
As we have already mentioned, $k_{RP}$ is small in practise and there is not much differences between different datasets. \cite{venkatasubramanian2011} suggested that $k_{RP}=2\ln{n}/0.25^2$ which is just about 500 for a dataset of ten millions points. We conducted an experiment with different $k_{RP}$ in each dataset. The results in Figure \ref{fig:kRP} show that the parameter $k_{RP}$ is quite small since the accuracy curve is flat when $k_{RP}$ reaches a certain value (other datasets also have similar tendency). It shows that our $k_{RP}=50$ was suitable for the datasets in the experiments. Moreover, experiments in last sections show that the graph creation is the most dominant step and the running time of CESC is significantly faster than all the others. Therefore, $k_{RP}$ can be quite small and does not considerably affect the running time of CESC. This is another advantage of CESC since it is not sensitive to the parameters in terms of both accuracy and performance. For sampling based methods, the selection of the number of representatives to balance between accuracy and speed is not trivial.

\begin{figure*}
	\centering
    \subfloat[Spambase]{\label{fig:spam_kRP}\includegraphics[width=0.3\textwidth]{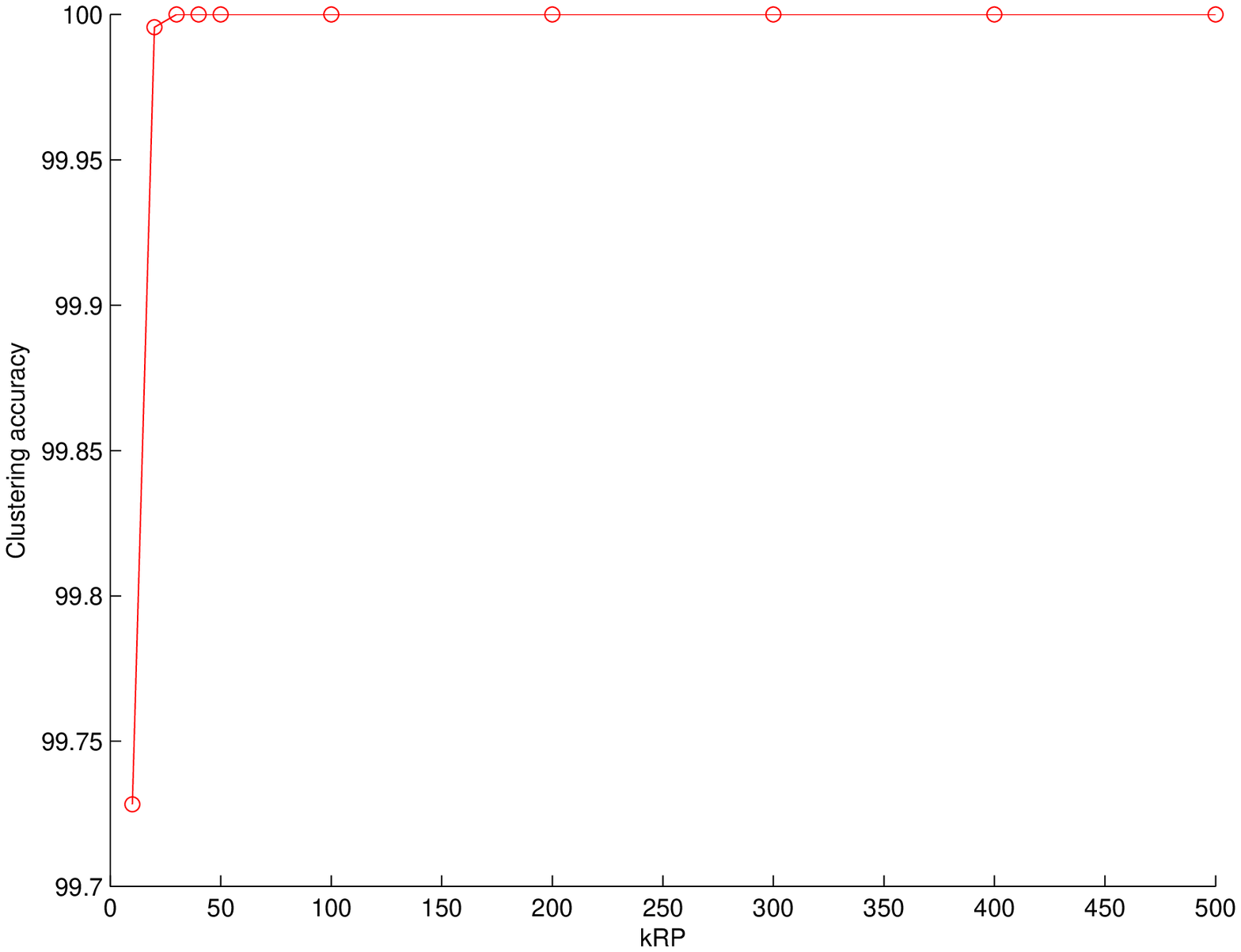}}
    \subfloat[Musk]{\label{fig:musk_kRP}\includegraphics[width=0.3\textwidth]{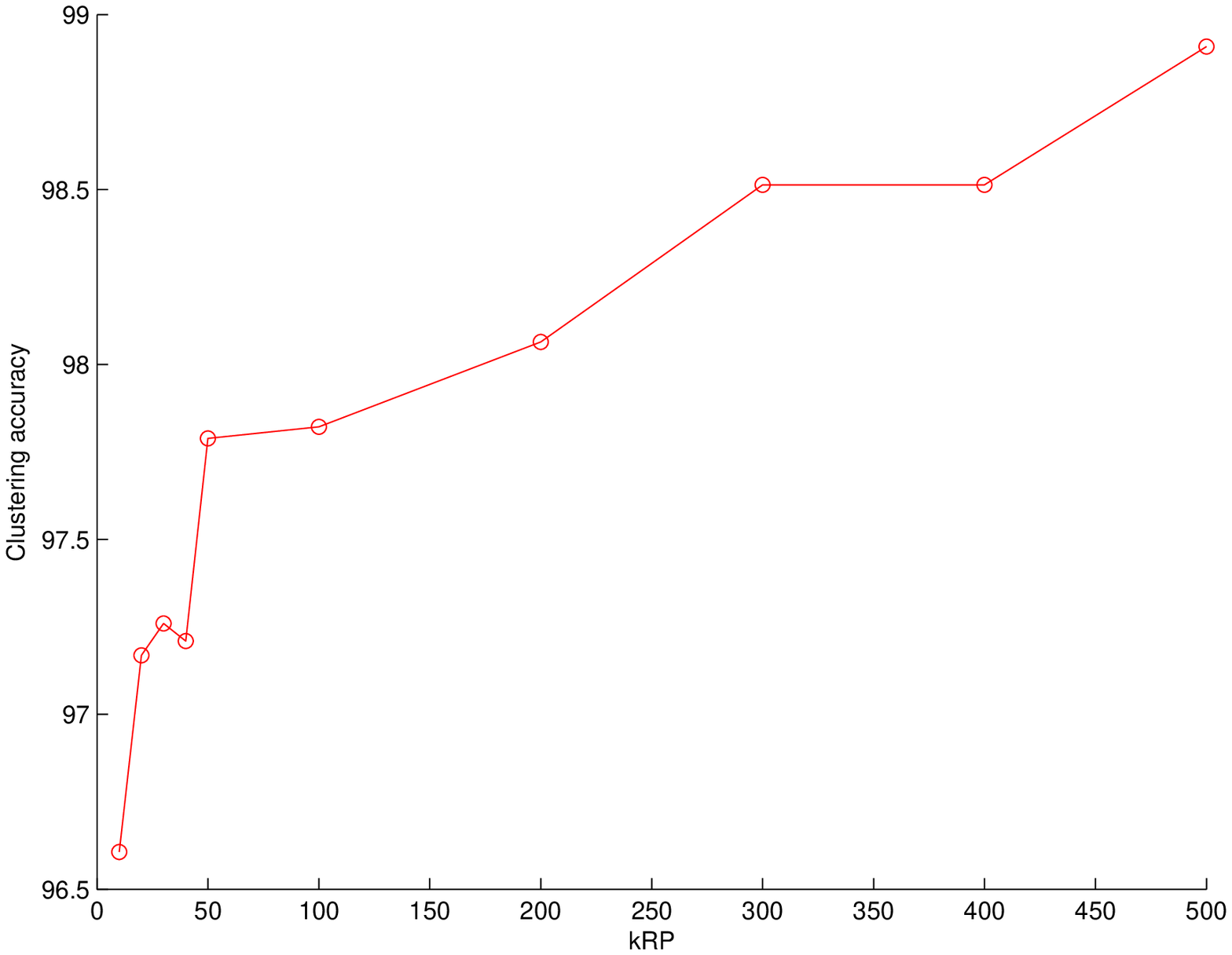}}
    \subfloat[Pen Digits]{\label{fig:pendigits_kRP}\includegraphics[width=0.3\textwidth]{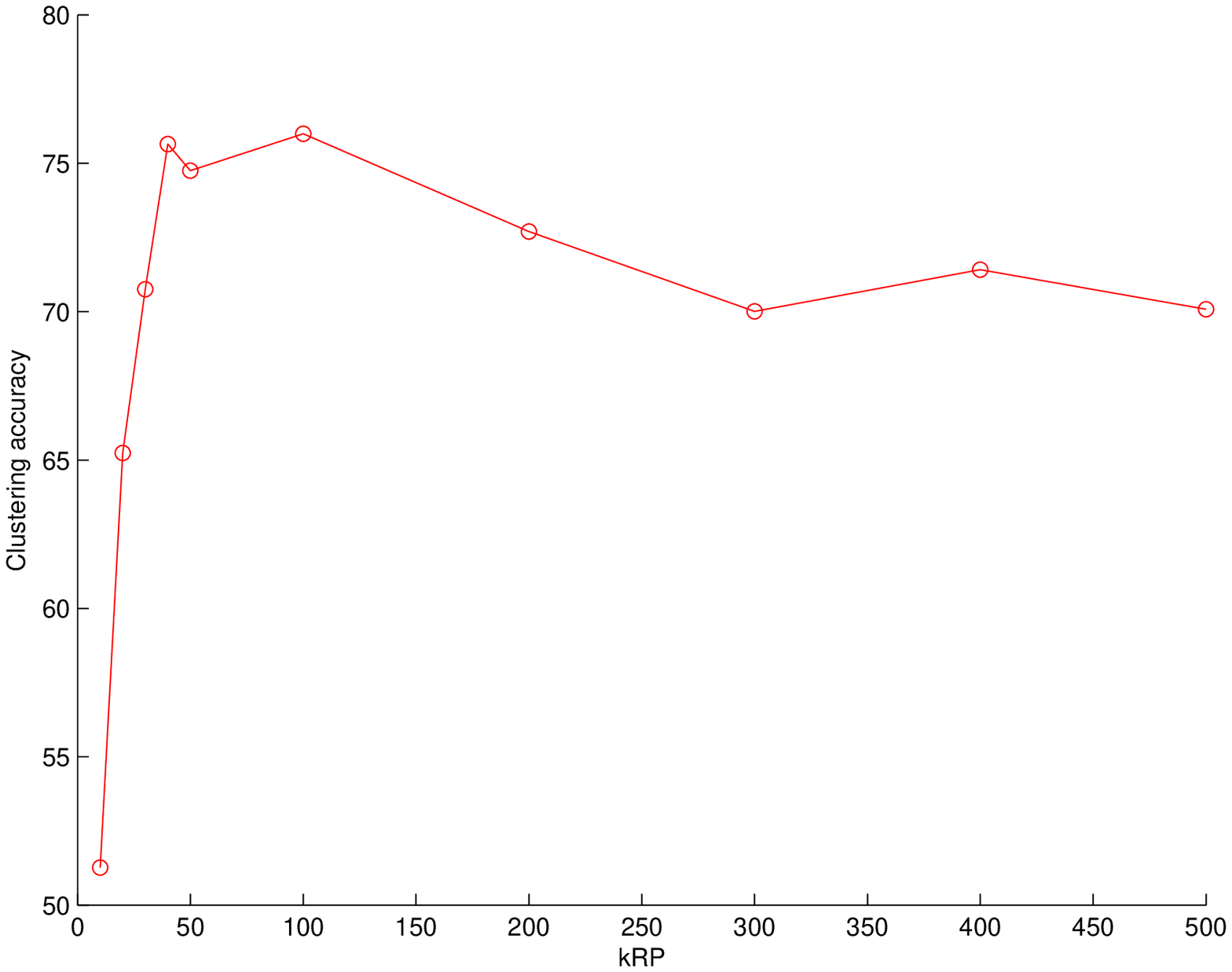}}
    \caption{$k_{RP}$ can be quite small since the accuracy curve just slightly changes when $k_{RP}$ reaches a certain value.}
	\label{fig:kRP}
\end{figure*}


\subsection{Graph Datasets}
One more advantage of CESC over KASP, Nystr\"{o}m, and LSC is that it can work directly on the similarity graph while the others cannot since they have a sampling step on the original feature data. An experiment to show the scalability of the proposed method in large graphs was conducted in DBLP co-authorship network obtained from \emph{http://dblp.uni-trier.de/xml/} and some real network graphs obtained from the Stanford Large Network Dataset Collection which is available at \emph{http://snap.stanford.edu/data/}. \emph{CA-AstroPh} is a collaboration network of Arxiv Astro Physics; \emph{Email-Enron} is an email communication network from Enron company; and \emph{RoadNet-TX} is a road network of Texas in the US.

All the graphs were undirected. The largest connected component was extracted if a graph data was not connected. We arbitrarily chose 50 as the number of clusters for all the datasets. The results using CESC are shown in Table \ref{tab:table6}.

\begin{table*}
  \centering
  \caption{The clustering time (second) for some network graphs. CESC took less than 10 minutes to create an approximate embedding for the network graph of more than 1.3 million nodes.}
  \begin{tabular}{| l | r |r|r|r|r|}
    \hline
    Dataset     & Nodes     & Edges     & Embedding & $k$-means & Total time (s)    \\
    \hline
    CA-AstroPh  & 17,903    & 197,001   & 24.36     & 50.62     & 74.98         \\ 
    Email-Enron & 33,696    & 180,811   & 27.33     & 167.08    & 194.41        \\ 
    DBLP        & 612,949   & 2,345,178 & 764.04    & 4572.25   & 5336.31       \\ 
    RoadNet-TX  & 1,351,137 & 1,879,201 & 576.62    & 4691.53   & 5268.15       \\ 
    \hline
  \end{tabular}
  \label{tab:table6}
\end{table*}

In case of graph data, the running time of $k$-means was dominant the whole method. CESC took only less than 10 minutes to create an approximate embedding for the network graph of more than 1.3 million nodes.

\subsection*{DBLP case study}
Since all the above graphs are too big to do a qualitative analysis, a subset of main data mining conferences in the DBLP graph was analyzed. We selected only authors and publications appearing in KDD, PKDD, PAKDD, ICDM, and SDM. Each author also need to have at least 10 publications and his/her co-authors also need to have such minimum publications. This selected only authors who published highly in major data mining conferences and collaborated with the similar kind of co-authors. Then the biggest connected component of the graph was extracted. The final graph has 397 nodes and 1,695 edges.

CESC was applied to the subgraph with $k=50$ clusters. Since researchers have collaborated and moved from research groups to research groups overtime, some clusters are probably a merge of groups caused by the collaborations and movings of prominent researchers. However, the method can effectively capture clusters representing many well known data mining research groups in CMU, IBM Research Centers (Watson and Almaden), University of California Riverside, LMU Munich, University of Pisa, University of Technology Sydney, Melbourne University, etc.

\section{Discussion}
\label{chapter:discussion}
Von Luxburg, Radl, and Hein in their paper \cite{luxburg2010} showed that the commute time between two nodes on a random geometric graph converges to an expression that only depends on the degrees of these two nodes and does not take into account the structure of the graph. Therefore, they claimed that it is meaningless as a distance function on large graph. However, their results do not reject our work because of the following reasons.
\begin{itemize}
    \item Their proof was based on random geometric graphs which may not be the case in practise. The random geometric graph does not have natural clusters which clustering algorithms try to detect. Moreover, there were many assumptions for the graph so that their claim can hold.
    \item Their experiments showed that the approximation becomes worse when the data has cluster structure. However, the condition for an unsupervised distance-based technique can work well is the data should have a cluster structure so that the separation based on distance is meaningful. We believe that many real datasets should have cluster structures in a certain degree.
    \item Our experiments show that CESC had a good approximation to spectral clustering and thus is not meaningless in several real datasets. It shows that approximate commute time embedding method can still be potential for using as a fast and accurate approximation of spectral clustering.
\end{itemize}

As already mentioned in the experiments of real feature data and graph data, CESC has the bottleneck at the creation of the nearest neighbor graph and $k$-means algorithm. The cost to create the embedding is actually very small comparing to the whole cost of the algorithm. Once we have the embedding, we can choose any fast partition or hierarchical clustering techniques to use on that. \cite{chen2010} proposed methods to improve the cost of creating the nearest neighbor graph and $k$-means in both centralized and distributed manners. Therefore, we believe CESC can be improved a lot more using these techniques. However, it is beyond the scope of this work.

\section{Conclusion}
\label{chapter:conclusion}
The paper shows the clustering using approximate commute time embedding is a fast and accurate approximation for spectral clustering. The strength of the method is that it does not involve any sampling technique which may not correctly represent the whole dataset. It does not need to use any eigenvector as well. Instead it uses the random projection and a linear time solver which guarantee its accuracy and performance. The experimental results in several synthetic and real datasets and graphs with various sizes show the effectiveness of the proposed approaches in terms of performance and accuracy. It is faster than the state-of-the-art approximate spectral clustering techniques while maintaining better clustering accuracy. The proposed method can also be applied directly to graph data. It takes only less than 10 minutes to create the approximate embedding for a network graph of more than 1.3 million nodes. Moreover, once we have the embedding, the proposed method can be applied to any application which utilize the commute time such as image segmentation, anomaly detection, and collaborative filtering.

In the future, techniques to avoid the bottleneck of CESC including the acceleration of the graph creation and $k$-means will be investigated. Moreover, though the analysis and experimental results show that CESC and spectral clustering have quite similar clustering ability, a deeply theoretical analysis need to be done to examine the strength and weakness of each method against the other.

\section*{Acknowledgement}
The authors of this paper acknowledge the financial support of the Capital Markets CRC.

\bibliographystyle{plain}
\bibliography{CommuteDistance_SDM2012}

\end{document}